%% file: exp_fam.tex
\documentclass{article}
\usepackage{nips07submit_e,times}

\usepackage{natbib}
\usepackage{color}
\usepackage{latexsym}
\usepackage{amsmath,amssymb,amsthm,amsfonts}
\usepackage{algorithm,algorithmic}

\include{notation}

\title{
Learning Exponential Families in High-Dimensions:
\vspace*{0.05in} \\ 
\Large{Strong Convexity and Sparsity}
}

\author{ Sham M. Kakade
}

\author{
Sham M. Kakade \\
Department of Statistics\\
The Wharton School\\
University of Pennsylvania, USA
\And
Ohad Shamir\\
School of Computer Science and Engineering\\
The Hebrew University of Jerusalem, Israel
\And
Karthik Sridharan\\
Toyota Technological Institute\\
Chicago, USA
\And
Ambuj Tewari\\
Toyota Technological Institute\\
Chicago, USA
}

\begin{document}
%\makeanontitle
\maketitle
%\vspace*{-.75cm}\begin{center}
%	{\bf DRAFT. PLEASE DO NOT DISTRIBUTE.}
%\end{center}
\vspace*{.75cm}
\begin{abstract}
The versatility of exponential families, along with their attendant
convexity properties, make them a popular and effective
statistical model. A central issue is learning these models in
high-dimensions, such as when there is some sparsity pattern of the
optimal parameter. This work characterizes a certain strong convexity property of \emph{general} exponential families,
which allow their generalization
ability to be quantified. In particular, we show how this property can be
used to analyze generic exponential families under $L_1$
regularization.
\end{abstract}

\input{intro}

\input{setting}

\input{strong}

\input{sparsity}

\input{2ssparse}

% \newpage 

\bibliographystyle{abbrv}
\bibliography{mybib,bib,sparse}

\newpage

\input{appendix}

\newpage
\input{errata}

\end{document}

%% file: notation.tex
\newtheorem{theorem}{Theorem}[section]

\newtheorem{lemma}[theorem]{Lemma}
\newtheorem{corollary}[theorem]{Corollary}
\newtheorem{proposition}[theorem]{Proposition}
\newtheorem{definition}[theorem]{Definition}
\newtheorem{assumption}[theorem]{Assumption}

\newcommand{\argmin}{{\rm argmin }}

\newcommand{\inner}[1]{\langle #1 \rangle}
\newcommand{\OneNorm}[1]{\| #1 \|_1}
\newcommand{\TwoNorm}[1]{\| #1 \|_2}
\newcommand{\InfNorm}[1]{\| #1 \|_\infty}
\newcommand{\FishNorm}[1]{\| #1 \|_\sFisher}

\newcommand{\R}{{\mathbb{R}}}
\newcommand{\Expct}{{\mathbb E}}

\newcommand{\Loss}{{\cal L}}
\newcommand{\ELoss}{{\widehat{\cal L}}}

\newcommand{\sFisher}{{\mathcal{F}^\star}}
\newcommand{\stheta}{{\theta^\star}}
\newcommand{\htheta}{\hat\theta}

\newcommand{\moment}{{m}}
\newcommand{\cumulant}{{c}}

\newcommand{\staty}{{t}}

\newcommand{\smoment}{{\alpha}}
\newcommand{\salpha}{{\alpha^\star}}

\newcommand{\eigmin}{{\kappa^\star_{\mathrm{min}}}}
\newcommand{\eigmax}{{\kappa^\star_{\mathrm{max}}}}

\newcommand{\Sup}{{S}}
\newcommand{\SupC}{{S^C}}
\newcommand{\SupSize}{{s}}

\newcommand{\dist}{{\rho}}
\newcommand{\thresh}{\tau}
\newcommand{\Clip}{\thresh}

\newcommand{\cliptheta}{\hat{\theta}^{\thresh}}

\newcommand{\reftheta}{{\tilde\theta}}

%% file: intro.tex
\section{Introduction}

Exponential models are perhaps the most versatile and pragmatic
statistical model for a variety of reasons --- modelling flexibility
(encompassing discrete variables, continuous variables, covariance
matrices, time series, graphical models, etc); convexity properties allowing ease of
optimization; and robust generalization ability.
%a natural building block in graphical models; maximum
%entropy justifications. 
A principal issue for applicability to large
scale problems is estimating these models when the ambient dimension
of the parameters, $p$, is much larger than the sample size $n$ ---
the ``$p\gg n$'' regime.

Much recent work has focused on this problem in the special case of
linear regression in high dimensions, where it is assumed that the
optimal parameter vector is sparse
(e.g. \cite{ZhaoYu06,candes-2007-35,MeinshausenYu,bickel-2008}). This
body of prior work focused on: sharply characterizing the convergence
rates for the prediction loss; consistent model selection; and
obtaining sparse models. As we tackle more challenging problems, there
is a growing need for model selection in more general exponential
families. Recent work here includes learning Gaussian graphs
(\cite{Ravikumar:1141507}) and Ising models (\cite{Ravikumar:Ising}).

Classical results established that consistent estimation in
\emph{general} exponential families is possible, in the asymptotic
limit where the number of dimensions is held constant (though some
work establishes rates under certain conditions as $p$ is allowed to
grow slowly with $n$ \citep{portnoy,ghosal}). However, in modern
problems, we typically grow $p$ rapidly with $n$ (so even asymptotically we
are often interested in the regime where $p\gg n$, as in the case of
sparse estimation). While we have a handle on this question for a
variety of special cases, a pressing
question here is understanding how fast $p$ can scale as a
function of $n$ in \emph{general} exponential families --- such an
analysis must quantify the relevant aspects of the particular family
at hand which govern their convergence rate. This is the focus of this
work. We should emphasize that throughout this paper, while we are
interested in \emph{modelling} with an exponential family, we are
agnostic about the true underlying distribution (e.g we do not
necessarily assume that the data generating process is from an
exponential family).

\paragraph{Our Contributions and Related Work}

The key issue in analyzing the convergence rates of
exponential families in terms of their prediction loss (which we take to
be the log loss) is in characterizing the nature in which they are
strictly convex --- roughly speaking, in the asymptotic regime where
we have a large sample size $n$ (with $p$ kept fixed), we have
a central limit theorem effect where the log loss of any
exponential family approaches the log loss of a Gaussian, with a
covariance matrix corresponding to the Fisher information matrix. Our
first main contribution is quantifying the rate at which this 
effect occurs in general exponential families.

In particular, we show that every exponential family satisfies a
certain rather natural growth rate condition on their standardized
moments and standardized cumulants (recall that the $k$-th
standardized moment is the \emph{unitless} ratio of the $k$-th central
moment to the $k$-th power of the standard deviation, which for
$k=3,4$ is the skew and kurtosis). This condition is rather mild,
where these moments can grow as fast as $k!$. Interestingly, similar
conditions have been well studied for obtaining exponential tail
bounds for the convergence of a random variable to its
mean~\citep{Bernstein46}. We show that this
growth rate characterizes the rate at which the prediction loss of the
exponential family behaves as a strongly convex loss function. In
particular, our analysis draws many parallels to that of the analysis
of Newton's method, where there is a ``burn in'' phase in which a
number of iterations must occur until the function
behaves as a locally quadratic function --- in our statistical
setting, we now require a (quantified) ``burn in'' sample size, where
beyond this threshold sample size, the prediction loss inherits the
desired strong convexity properties (i.e. it is locally quadratic).

Our second contribution is an analysis of $L_1$ regularization in
generic families, in terms of both prediction loss and the sparsity
level of the selected model. Under a particular sparse eigenvalue
condition on the design matrix (the Restricted Eigenvalue (RE)
condition in \cite{bickel-2008}), we show how $L_1$ regularization in
general exponential families enjoys a convergence rate of $O(\frac{s
  \log p}{n})$ (where $s$ is the number of relevant features). This RE
condition is one of the least stringent conditions which permit 
this optimal convergence rate for linear regression case (see
\cite{bickel-2008}) --- stronger mutual incoherence/irrepresentable
conditions considered in \cite{ZhaoYu06} also provide this rate. We
show that an essentially identical convergence rate can be achieved
for \emph{general} exponential families --- our results are non-asymptotic and
precisely relate $n$ and $p$.

Our final contribution is one of \emph{approximate} sparse model
selection, i.e. where our goal is to obtain a sparse model with low
prediction loss. A drawback of the RE condition in comparison to the
mutual incoherence condition is that the latter permits perfect
recovery of the true features (at the price of a more stringent
condition). However, for the case of the linear regression,
\cite{ZhaoYu06,bickel-2008} show that, under a sparse
eigenvalue or RE condition, the
$L_1$ solution is actually sparse itself (with a
multiplicative increase in the sparsity level, that depends on a certain condition number of the design matrix) -- so while the
the $L_1$ solution may not precisely recover the true model, it still
is sparse (with some multiplicative increase) and does recover those
features with large true weights.

For general exponential families, while we do not have a
characterization of the sparsity level of the $L_1$-regularized
solution (an interesting open question), we do however provide a
simple two stage procedure (thresholding and refitting)
which provides a sparse model, with support on no more than merely $2s$
features and which has nearly as good performance (with a rather
mild increase in the risk) --- this result is novel even for the
square loss case. Hence, even under the rather
mild RE condition, we can obtain both a favorable convergence rate and a 
 sparse model for generic families.

%% file: setting.tex
\section{The Setting}

Our samples $\staty \in \R^p$ are distributed independently according
to $D$, and we model the process with $P(\staty|\theta)$, where $\theta\in\Theta$. However, we do
not necessarily assume that $D$ lies in this model class. The class of
interest is \emph{exponential families}, which, in their natural form,
we denote by:
\[
P(\staty|\theta) = h_\staty \exp\{\inner{\theta, \staty} - \log Z(\theta)\}
\]
where $\staty$ is the natural sufficient statistic for $\theta$, and
$Z(\theta)$ is the partition function. Here, $\Theta$ is the
natural parameter space --- the (convex) set where $Z(\cdot)$ is finite.
While we work with an exponential family in this general
(though natural) form, it should be kept in mind that $t$ can
be the sufficient statistic for some prediction variable $y$ of
interest, or, for a generalized linear model (such as for logistic or linear regression),
we can have $t$ be a function of both $y$ and some covariate $x$ (see
\cite{Dobson90}). We return to this point later.

%We denote the log partition function by
%\[
%\logpart(\theta) = \log Z(\theta) \, .
%\]

Our prediction loss is the likelihood function and
$\theta^*$ is the optimal parameter, i.e.
\[
\Loss(\theta) = \Expct_{\staty \sim D}[-\log P(\staty|\theta)],
\quad \quad
\stheta = \argmin \ \Loss(\theta) \, .
\]
where the $\argmin$ is over the natural parameter space and it is
assumed that this $\theta^*$ is an interior point of this space. Later
we consider the case where $\stheta$ is sparse. 

We denote the Fisher information of $P(\cdot|\stheta)$ as 
$
\sFisher = \Expct_{\staty\sim P(\cdot|\stheta)}\left[ - \nabla^2 \log P(\staty|\stheta) \right]
$, under the model of
$\stheta$. The induced ``Fisher risk''  is 
\[
\FishNorm{\theta-\stheta}^2 := (\theta-\stheta)^\top\sFisher (\theta-\stheta) \ .
\]
We also consider the $L_1$ risk $\OneNorm{\theta-\stheta}$. 

For a sufficiently large sample size, we expect that the Fisher risk
of an empirical minimizer $\htheta$, $\FishNorm{\htheta-\stheta}^2$,
be close to $\Loss(\htheta)-\Loss(\stheta)$ --- one of our main
contributions is quantifying when this occurs in general exponential
families. This characterization is then used to quantify the
convergence rate for $L_1$ methods in these families. We
also expect this strong convexity property to be useful for
characterizing the performance of other regularization methods as
well.

All proofs can be found in the appendix.

%% file: strong.tex
\section{(Almost) Strong Convexity of Exponential Families}\label{section:strong}

We first consider a certain bounded
growth rate condition for standardized moments and standardized
cumulants, satisfied by all exponential families. This growth rate is fundamental in establishing how fast
the prediction loss behaves as a quadratic function. Interestingly,
this growth rate is analogous to those conditions used for obtaining
exponential tail bounds for arbitrary random variables.

\subsection{Analytic Standardized Moments and Cumulants}

\paragraph{Moments:}
For a univariate random variable $z$ distributed by $\dist$, let us
denote its $k$-th central moment (centered at the mean) by:
\[
\moment_{k,\dist}(z) = \Expct_{z \sim \dist}
\left[z-\moment_{1,\dist}(z)\right]^k 
\]
where $\moment_{1,\dist}(z)$ is the mean $\Expct_{z\sim
  \dist}[z]$.  Recall that the $k$-th standardized moment is the ratio of the $k$-th central moment
to the $k$-th power of the standard deviation,
i.e. $\frac{\moment_{k,\dist}(z)}{\moment_{2,\dist}(z)^{k/2}}$.  This
normalization with respect to standard
deviation makes the standard moments unitless quantities. For $k=3$
and $k=4$,
the standardized moments are the skew and kurtosis.

%As we shall see, the behavior of these standardized moments are a key
%concept in understanding certain convexity properties of general
%exponential families. 

We now define the \emph{analytic} standardized
moment for  $z$ ---
we use the term analytic
to reflect that if the moment generating function of $z$ is analytic\footnote{Recall that a real valued
function is analytic on some domain of $\R^p$ if the derivatives of
all orders exist, and if for each interior point,
the Taylor series converges in some sufficiently small neighborhood of
that point.} then $z$ has an analytic moment.

\begin{definition}
Let $z$ be a univariate random variable under $\dist$. Then $z$ has an \emph{analytic
standardized moment} of $\smoment$ if the standardized
moments exist and are bounded as follows:
\[
\forall k\geq 3, \ \
\left| \frac{\moment_{k,\dist}(z)}{\moment_{2,\dist}(z)^{k/2}} \right|
\leq  \frac{1}{2} k! \ \smoment^{k-2}
\]
(where the above is assumed to hold if the denominator is $0$). If
$t\in\R^p$ is a multivariate random variable distributed according 
to $\dist$, we say that $t$ has an analytic standardized moment of
$\smoment$ with respect to a subspace $\mathcal{V}\subset\R^p$ (e.g. a set of directions)
if the above bound holds for all univariate
$z=\inner{v,t}$ where $v\in \mathcal{V}$.
\end{definition}

This condition is rather mild in that the standardized
moments increase as fast as $k! \alpha^{k-2}$ (in a sense $\alpha$ is
just a unitless scale, and it is predominantly the $k!$ which makes
the condition rather mild). This condition is closely related
to those used in obtaining sharp exponential type tail bounds for the
convergence of a random variable to its mean --- in particular, the
Bernstein conditions~\citep{Bernstein46} are almost identical to the above,
expect that they use the $k$-th raw moments (not central
moments)~\footnote{The Bernstein inequalities used in deriving tail
  bounds require that, for all $k\geq 2$, $\frac{\Expct[z^k]}{\Expct[z^2]}\leq
  \frac{1}{2}k!L^{k-2}$ for some constant $L$ (which has units of
  $z$).  }. In fact, these moment
conditions are weaker than requiring ``sub-Gaussian'' tails.

While we would not expect analytic moments to be finite for all
distributions (e.g. heavy tailed ones), we will see that exponential families have (finite)
analytic standardized moments.

%, which is a direct consequence from
%that the moment generating function of an exponential family is analytic on
%its natural parameter space.  

\paragraph{Cumulants:} Recall that the cumulant-generating function $f$
of $z$ under $\dist$ is the log of the moment-generating function, if it exists,
i.e. $f(s) = \log \Expct [e^{sz}]$.  The $k$-th cumulant is given by
the $k$-th derivate of $f$ at $0$, i.e. $\cumulant_{k,\dist}(z) =
f^{(k)}(0)$.  The first, second, and third cumulants are just the
first, second, and third central moments --- higher cumulants are
neither moments nor central moments, but rather more complicated
polynomial functions of the moments (though these relationships are
known). Analogously, the $k$-th standardized cumulant is
$\frac{\cumulant_{k,\dist}(z)}{\cumulant_{2,\dist}(z)^{k/2}}$ --- this
normalization with respect to standard deviation (the second cumulant
is the variance) makes these unitless quantities.

Cumulants are viewed as equally fundamental as central moments, and we
make use of their behavior as well --- in certain settings, it is more
natural to work with the cumulants.  We define the \emph{analytic}
standardized cumulant analogous to before:

\begin{definition}
Let $z$ be a univariate random variable under $\dist$. Then $z$ has an \emph{analytic
standardized cumulant} of $\smoment$ if the standardized
cumulants exist and are bounded as follows:
\[
\forall k\geq 3, \ \
\left| \frac{\cumulant_{k,\dist}(z)}{\cumulant_{2,\dist}(z)^{k/2}} \right|
\leq  \frac{1}{2} k! \ \smoment^{k-2}
\]
(where the above is assumed to hold if the denominator is $0$). If
$t\in\R^p$ is a multivariate random variable distributed according
to $\dist$, we say that $t$ has an analytic standardized cumulant of
$\smoment$ with respect to a subspace $\mathcal{V}\subset\R^p$
if the above bound holds for all univariate
$z=\inner{v,t}$ where $v\in \mathcal{V}$.
\end{definition}

\paragraph{Existence:} The following lemma shows that exponential families have (finite)
analytic standardized moments and cumulants, as a consequence of the analyticity
of the moment and cumulant generating functions (the proof is in the appendix).

\begin{lemma} \label{lemma:exist}
If $\staty$ is the sufficient statistic of an exponential family with parameter $\theta$, where $\theta$ is an interior point of the natural parameter space, then
$\staty$ has both a finite analytic
standardized moment and a finite analytic standardized cumulant, with respect to all directions in $\R^p$.
\end{lemma}

\subsection{Examples}

Let us consider a few examples. Going through them, there are two issues to bear in mind. First, $\alpha$ is quantified only at a particular $\theta$
(later, $\stheta$ is the point we will be interested in) --- note that we
do not require any uniform conditions on any derivatives over all
$\theta$. Second, we are interested in how $\alpha$ could depend on
the dimensionality --- in some cases, $\alpha$ is dimension free and
in other cases (like for generalized linear models), $\alpha$ depends
on the dimension through spectral properties of $\sFisher$ (and this
dimension dependence can be relaxed in the sparse case that we
consider, as discussed later).

\subsubsection{One Dimensional Families}\label{sec:onedim}

When $\theta$ is a scalar, there is no direction $v$ to consider.

\paragraph{Bernoulli distributions} In the canonical form, the Bernoulli distribution is,
\[
	P(y| \theta) = \exp\left( y \theta - \log ( 1+e^\theta) \right)
\]
with $\theta \in \R = \Theta$. We have $\moment_1(\stheta) = e^\stheta/(1+e^\stheta)$. The central moments satisfy $\moment_2(\stheta) = \moment_1(\stheta)(1-\moment_1(\stheta))$ and $\moment_k(\stheta) \le \moment_2(\stheta)$ for $k\ge 3$. Thus, $\smoment = 1/\sqrt{\moment_2(\stheta)}$ is a standardized analytic moment at any $\stheta \in \Theta$. Further, $\cumulant_k(\stheta) \le \cumulant_2(\stheta) = \moment_2(\stheta)$ for $k \ge 3$. Thus, $\smoment$ is also a standardized analytic cumulant at any $\stheta \in \Theta$.

\paragraph{Unit variance Gaussian distributions} In the canonical form, unit variance Gaussian is,
\[
	P(y| \theta) = \exp\left(-\frac{y^2}{2}\right) \exp\left( y \theta - \frac{\theta^2}{2} \right)
\]
with $\theta \in \R = \Theta$. We have $\moment_1(\stheta) = \stheta$
and $\moment_2(\stheta) = 1$. Odd central moments are $0$ and for even
$k\ge 4$, we have $\moment_k(\stheta) = \frac{k!}{2^{k/2}
  (k/2)!}$. Thus, $\smoment = 1$ is a standardized analytic moment at
any $\stheta \in \Theta$. However, the log-likelihood is already
quadratic in this case (as we shall see, there should be no
``burn in'' phase until it begins to look like a quadratic!). This
becomes evident if we consider the cumulants instead. All cumulants
$\cumulant_k(\stheta) = 0$ for $k \ge 3$ and hence $\smoment = 0$ is a
standardized analytic cumulant at any $\stheta \in \Theta$ ---
curiously, cumulant generating function cannot be a finite order
polynomial of order greater than 2. 

\subsubsection{Multidimensional Gaussian Covariance Estimation
  (i.e. ``Gaussian Graphs'')}\label{sec:gg}

\newcommand{\preci}{\Theta}

\newcommand{\eye}{\mathbf{I}}

Consider a mean zero $p$-dimensional multivariate Normal parameterized by the precision matrix $\preci$,
\[
	P( Y | \preci ) = \frac{1}{(2\pi)^{p/2}}\exp \left( -\frac{1}{2} \inner{ \Theta, Y Y^\top } + \log \det(\Theta) \right)\ .
\]
A ``direction'' here is a positive semi-definite (p.s.d.)
matrix $V$, and we seek the cumulants of the random variable
$ \inner{ V, YY^\top }$ .

Note that $Y Y^\top$ has Wishart distribution $W_p(\preci^{-1},1)$ with the moment generating function,
\[
	V \mapsto \Expct\left[ \exp\left( \inner{V, YY^\top} \right) \right] = \det\left( \eye - 2 V \preci^{-1} \right)^{-1/2} \ .
\]
Let $\lambda_i$'s be the eigenvalues of $V \preci^{-1}$. Then, taking logs, the cumulant generating function $f(s)$,
\[
	f(s) = \log \Expct\left[ \exp(s\inner{V, YY^\top})\right] = \log \prod_{i=1}^p (1 - 2s \lambda_i)^{-1/2} = \frac{-1}{2} \sum_{i}^p \log(1-2s \lambda_i)\ .
\]
The $k$th derivative of this is
\[
	f^{(k)}(s) = \frac{1}{2} \sum_{i=1}^p \frac{(k-1)! (2\lambda_i)^k}{(1-2s\lambda_i)^k}\ .
\]
Thus, the cumulant $\cumulant_{k,\Theta}(V) = f^{(k)}(0) = 2^{k-1}(k-1)! \sum_i \lambda_i^k$. Hence, for $k \ge 3$,
\[
	\frac{\cumulant_{k,\Theta}(V)}{(\cumulant_{2,\Theta}(V))^{k/2}} = \frac{2^{k-1}(k-1)!\sum_i \lambda_i^k }{(2 \sum_i \lambda_i^2)^{k/2}} = \frac{1}{2} 2^{k/2} (k-1)! \frac{\sum_i \lambda^{(k/2)\cdot 2}}{(\sum_i \lambda_i^2)^{k/2}} \le \frac{1}{2} 2^{k/2-1} \cdot k! \ .
\]
Thus, $\smoment = \sqrt{2}$ is a standardized analytic cumulant at $\Theta$. Note that it is harder to estimate the central moments in this case. This example is also interesting in connection to the analysis of Newton's method as the function $\log \det(\Theta)$ is self-concordant on the cone of p.s.d. matrices. 

%\sk{make comment that this function is self concordant} - DONE!

\subsubsection{Generalized Linear Models}\label{sec:glm}

Consider the case where we have some covariate, response pair $(X,Y)$ drawn from some distribution $D$. Suppose that we have a family of distributions $P(\cdot|\theta; X)$ such that, for each $X$, it is an exponential family with natural sufficient statistic $t_{y,X}$,  
\[
	P(y | \theta; X) = h_y \exp\left( \inner{ \theta, \staty_{y,X} } - \log Z_X(\theta) \right)\ ,
\]
where $\theta \in \Theta$. The loss we consider is
$
	\Loss(\theta) = \Expct_{X,Y \sim D}\left[ -\log P(y | \theta; X) \right]
$. A special case of this setup is as follows. Say we have a one dimensional exponential family
\[
	q_\nu(y) = h_y \exp( y \nu - \log Z(\nu) )\ ,
\]
where $y, \nu \in R$. The family $P(\cdot| \theta; X)$ can be
be simply $q_{\inner{\theta, X}}$ (i.e. taking $\nu = \inner{\theta, X}$). Thus, 
\[
	P(y | \theta; X) = h_y \exp\left( y \inner{ \theta, X } - \log Z( \inner{\theta, X} ) \right)\ . 
\]
We see that $\staty_{y,X} = yX$ and $Z_X(\theta) = Z(\inner{\theta,
  X})$. For example, when $q_\nu$ is either the Bernoulli family or
the unit variance Gaussian family, this corresponds to {\em logistic regression} or {\em least squares regression}, respectively. It is easy to see that the analogue of having a standardized analytic moment of $\smoment$ at $\theta$ w.r.t. a direction $v$ is to have
$$
	\frac{\moment_{k,\theta}(v)}{(\moment_{2,\theta}(v))^{k/2}} \le \frac{1}{2} k! \smoment^{k-2}\ ,
$$
where
\begin{align*}
\moment_{k,\theta}(v) &= \Expct_{X} \left[ \moment_{k,P(\cdot|\theta;X)}(\inner{\staty_{y,X}, v}) \right]\ .
\end{align*}
In the above equation, the expectation is under $X \sim D_X$, the marginal of $D$ on $X$. If the sufficient statistic $\staty_{y,X}$ is bounded by $B$ in the $L_2$ norm a.s. and the expected Fisher information matrix
\[
	\Expct_{X}\left[ \Expct_{y\sim P(\cdot|\theta;X)} \left[ -\nabla^2\log P(y|\theta; X) \right] \right]
\]
has minimum eigenvalue $\lambda_\mathrm{min}$, then we can choose $\smoment = B/\lambda_\mathrm{\min}$. Note that $\lambda_\mathrm{\min}$ could be small but it arose only because we are considering an arbitrary direction $v$. If the set of directions $\mathcal{V}$ is smaller, then we can often get less pessimistic bounds. For example, see section \ref{sec:glmsparse} in the appendix. We also note that similar bounds can be derived when we assume subgaussian tails for $\staty_{y,X}$ rather than assuming it is bounded a.s.

\subsection{Almost Strong Convexity}

Recall that a strictly convex function $F$ is strongly
convex if the Hessian of $F$ has a (uniformly) lower bounded eigenvalue
(see~\cite{BoydVa04}). Unfortunately, as for all strictly convex functions, exponential families only
behave in a strongly convex manner in a (sufficiently small)
neighborhood of $\stheta$. Our first main result quantifies when
this behavior is exhibited.

\begin{theorem} \label{thm:strong} 
(Almost Strong Convexity) Let $\smoment$ be either the analytic standardized moment or cumulant under $\stheta$ with respect to
a subspace $\mathcal{V}$. For any $\theta$ such that $\theta-\stheta\in\mathcal{V}$, if either
\[
\Loss(\theta)-\Loss(\stheta) \leq \frac{1}{65 \smoment^2}
\ \quad \mathrm{or} \quad \
\FishNorm{\theta-\stheta}^2 \leq \frac{1}{16\smoment^2}
\]
then
\[
\frac{1}{4} \FishNorm{\theta-\stheta}^2
\ \leq \
\Loss(\theta)-\Loss(\stheta) 
\ \leq \
\frac{3}{4} \FishNorm{\theta-\stheta}^2
\]
\end{theorem}

Suppose $\theta$ is an MLE. Both preconditions can be
thought of as a ``burn in'' phase --- the idea being that
initially a certain number of samples is needed until the loss of $\theta$
is somewhat close to the minimal loss; after which point, the
quadratic lower bound engages. This is analogous
to the analysis of the Newton's method, which quantifies the number of
steps needed to enter the quadratically convergent phase (see
~\cite{BoydVa04}). The constants of
$1/4$ and $3/4$ can be made arbitrarily close to $1/2$ (with a longer ``burn in'' phase), as expected
under the central limit theorem.

A key idea in the proof is an expansion of the prediction regret in
terms of the central moments. We use the shorthand notation of
$\cumulant_{k,\theta}(\Delta)$ and $\moment_{k,\theta}(\Delta)$ to
denote the cumulants and moments of the random variable
$\inner{\Delta,t}$ under the distribution $P(\cdot|\theta)$.

\begin{lemma}\label{lemma:taylor}
(Moment and Cumulant Expansion) Define $\Delta=\theta-\stheta$. For all $s\in [0,1]$,
\begin{align*}
\Loss(\stheta+s\Delta)-\Loss(\stheta) & = 
\sum_{k=2}^\infty \frac{1}{k!} \cumulant_{k,\stheta}(\Delta) s^k\\
\Loss(\stheta+s\Delta)-\Loss(\stheta) & = 
\log\left( 1+
\sum_{k=2}^\infty \frac{1}{k!} \moment_{k,\stheta}(\Delta) s^k \right)
\end{align*}
where the equalities hold if the right hand sides converge. 
\end{lemma}

The proof of this Lemma (in the appendix) is relatively
straightforward. The key technical step in the proof of
Theorem~\ref{thm:strong} is characterizing when these expansions
converge. Note that
for $\Delta=\theta-\stheta$, even if $\FishNorm{\Delta}^2 \leq
\frac{1}{16\smoment^2}$ (one of our preconditions), a direct attempt at
lower bounding
$\Loss(\stheta+\Delta)-\Loss(\stheta)$ using the above expansions with the
analytic moment condition would not imply these expansions converge ---
the proof requires a more delicate argument.

%% file: sparsity.tex
\section{Sparsity}\label{section:sparse}

We now consider the case where $\theta^*$ is sparse, with support
$\Sup$ and sparsity level $\SupSize$, i.e. 
\[
\Sup=\{i:[\stheta]_i\neq 0 \} , \ \ \SupSize=|\Sup|
\]
In order to understand when $L_1$ regularized algorithms (for linear
regression) converge at a rate comparable to that of $L_0$ algorithms
(subset selection), \cite{MeinshausenYu} considered a sparse eigenvalue
condition on the design matrix, where the eigenvalues on any
small (sparse) subset are bounded away from 0.  \cite{bickel-2008}
relaxed this condition so that vectors whose support is ``mostly'' on
any small subset are not too small (see \cite{bickel-2008} for a
discussion). We also consider this relaxed condition, but now on the
Fisher matrix.

\begin{assumption}\label{ass:RE}
(Restricted Fisher Eigenvalues) For a vector $\delta$, let $\delta_\Sup$ be the vector such that $\forall i \in \Sup, [\delta_\Sup]_i = \delta_i$ and $\delta_\Sup$ is $0$ on the other coordinates, and let $\SupC$ denote the complement of $\Sup$. Assume that:
\begin{align*}
& \forall \delta \textrm{ s.t. } 
\OneNorm{\delta_{\SupC}} \leq 3 \OneNorm{\delta_{\Sup}} , 
\ \ \
\FishNorm{\delta} \geq \eigmin \TwoNorm{\delta_\Sup}\\
& \forall \delta \textrm{ s.t. } 
\delta_{\SupC}  = 0 , 
\ \ \
\FishNorm{\delta} \leq \eigmax \TwoNorm{\delta_\Sup}
\end{align*}
\end{assumption}

The constant of $3$ is for convenience. Note we only quantify on the
support $S$ --- a substantially weaker condition than in
\cite{MeinshausenYu,bickel-2008}, which quantify over \emph{all}
subsets (in fact, many previous algorithms/analysis actually use this
condition on subsets different from $S$,
e.g. \cite{MeinshausenYu,candes-2007-35,Zhang08}).

Furthermore, with regards to our analyticity conditions, our proof
shows that the subspace of
directions we need to consider is now restricted to the set:
\begin{align}\label{eq:resdir}
\mathcal{V} = \{v : \OneNorm{v_{\SupC}} \leq 3 \OneNorm{v_{\Sup}}\}
\end{align}
Under this Restricted Eigenvalue (RE) condition, we can replace the
minimal eigenvalue used in Example~\ref{sec:glm} by $\eigmin$ (section
\ref{sec:glmsparse} in appendix), which could be significantly smaller.

\subsection{Fisher Risk}

Consider the following regularized optimization problem:
\begin{equation}\label{eq:optimization}
\htheta = \argmin_{\theta\in\Theta} \ \widehat\Expct[-\log P(y|\theta)] + \lambda \OneNorm{\theta}
\end{equation}
where the empirical expectation is with respect to a sample. This
reduces to the usual linear regression example (for Gaussian means) and
involves the log-determinant in Gaussian graph setting (considered in \cite{Ravikumar:1141507}) where
$\theta$ is the precision matrix (see Example~\ref{sec:gg}).

Our next main result provides a risk bound, under the RE
condition. Typically, the regularization parameter $\lambda$ is specified as
a function of the noise level, under a particular
noise model (e.g. for linear regression case, where $Y=\beta X
+\eta$ with the noise model $\eta \sim
\mathcal{N}(0,\sigma^2)$, $\lambda$ is specified as $\sigma \sqrt{\frac{
    \log p}{n}}$ \citep{MeinshausenYu,bickel-2008}). Here, our
theorem is stated in a deterministic manner (i.e. it is a distribution
free statement), to explicitly show that an appropriate value of
$\lambda$ is determined by the $L_\infty$ norm of the measurement
error, i.e. $\InfNorm{\Expct[\staty] - \widehat \Expct[\staty]}$ ---
we then easily quantify $\lambda$ in a corollary under a mild
distributional assumption. Also, we must have that this measurement
error be (quantifiably) sufficiently small such that our ``burn
in'' condition holds.

\begin{theorem}\label{thm:risk}
(Risk) Suppose that Assumption \ref{ass:RE} holds and  $\lambda$ satisfies both
\begin{equation} \label{eq:lambda} 
\InfNorm{\Expct[\staty] - \widehat \Expct[\staty]}
\ \leq \
\frac{\lambda}{2} \ \ \ \textrm{ and }\ \ \ \lambda \leq \frac{1}{100  \salpha^2 \OneNorm{\stheta}}
\end{equation}
where $\salpha$ is the analytic standardized moment or cumulant of $\stheta$ for the subspace
$\mathcal{V}$ defined in \eqref{eq:resdir}. (Note this setting
requires that $\|\Expct[\staty] - \widehat \Expct[\staty]\|_{\infty}$
be 
sufficiently small).
Then if $\htheta$ is the solution to the optimization
problem in~\eqref{eq:optimization},  the Fisher
risk is bounded as follows
\begin{align*}
\frac{1}{4} \FishNorm{\htheta - \stheta}^2 & \leq \Loss(\htheta) - \Loss(\stheta) \le  \frac{9 \SupSize \lambda^2}{\eigmin^2} 
%\OneNorm{\htheta - \stheta} & \leq 
%\frac{6 s \lambda}{\eigmin^2} 
\end{align*}
and the $L_1$ risk is bounded as follows:
\begin{align*}
%\FishNorm{\htheta - \stheta}^2 & \leq 
%\frac{3s \lambda^2}{\eigmin^2} \\
\OneNorm{\htheta - \stheta} & \leq 
\frac{24 \SupSize \lambda}{\eigmin^2}  
\end{align*}
\end{theorem}

Intuitively, we expect the measurement error $\InfNorm{\Expct[\staty] - \widehat \Expct[\staty]}$ to be $O(\sigma \sqrt{\frac{\log
    p}{n}})$, so we think of $\lambda=O(\sigma \sqrt{\frac{\log
    p}{n}})$. Note this would recover the usual (optimal) risk bound of
$O(\sigma^2 \frac{s\log p}{n})$ (i.e. the same rate as an $L_0$
algorithm, up to the RE constant). Note that the mild dimension
dependence enters through the measurement error. Hence, our theorem
shows that \emph{all} exponential families exhibit favorable convergence
rates under the RE condition.

The following proposition and corollary quantify this under a mild
(and standard) distributional assumption (which can actually be
relaxed somewhat).

\begin{proposition}\label{prop:subgauss}
If $\staty$ is sub-Gaussian, ie. there exists $\sigma \ge 0$ such that
$\forall i$ and $\forall s \in \R$, $\mathbb{E}\left[ e^{s (\staty_i -
    \mathbb{E}\staty_i)} \right] \le e^{\sigma^2 s^2/2}$, then for any
$\delta > 0$, with probability at least $1-\delta$, 
\begin{align*}
\InfNorm{\Expct[\staty] - \widehat \Expct[\staty]} \le \sigma \sqrt{ \frac{ \log\left(\frac{p}{\delta}\right)}{n}}
\end{align*}
\end{proposition}

Bounded random variables are in fact sub-Gaussian (though
unbounded $t$ may also be sub-Gaussian, e.g. Gaussian random variables are obviously sub-Gaussian). The following corollary is immediate.

\begin{corollary}
Suppose the Assumption \ref{ass:RE} and the sub-Gaussian condition in Proposition \ref{prop:subgauss} hold.
For any $\delta > 0$, as long as $n \ge K \salpha^4
\OneNorm{\stheta}^2 \sigma^2 \log\left(\frac{p}{\delta}\right) $,
(where $K$ is a universal constant),  
setting $\lambda = 2 \sigma \sqrt{\frac{
    \log\left(\frac{p}{\delta}\right)}{n}}$, we have with probability at least $1-\delta$,
\begin{align*}
 \FishNorm{\htheta - \stheta}^2 & \leq   \left(\frac{36}{\eigmin^2}\right) \frac{\sigma^2 \SupSize \log\left(\frac{p}{\delta}\right)}{n} \ \ \ \textrm{ and }\ \ \  \OneNorm{\htheta - \stheta}  \leq 
\frac{48 \sigma  \SupSize}{\eigmin^2}  \sqrt{\frac{\log(\frac{p}{\delta})}{n}}
\end{align*}
\end{corollary}

%% file: 2ssparse.tex
\subsection{Approximate Model Selection} \label{section:2s} 

An important issue unaddressed by the previous result is the sparsity
level of our estimate $\htheta$. For the linear regression case,
\cite{MeinshausenYu,bickel-2008} show that the $L_1$ solution is
actually sparse, with a sparsity level of roughly
$O((\frac{\eigmax}{\eigmin})^2 s)$, (i.e. the sparsity level increases
by a factor which is essentially a condition number squared). In the
general setting, we do not have a characterization of the actual
sparsity level of the $L_1$ solution. 

However, we now present a two
stage procedure, which provides an estimate with support on merely
$2s$ features, with nearly as good risk (\cite{ShalevSrZh09_tech}
discuss this issue of trading sparsity for accuracy, but their results
are more applicable to settings with $O(\frac{1}{\sqrt{n}})$ rates.).
Consider the procedure where we select the set of coordinates which
have large weight under $\htheta$ (say greater than some threshold
$\thresh$). Then we refit to find an estimate with support only on
these coordinates. That is, we restrict our estimate to the
set $\Theta_\Clip = \{\theta \in \Theta : \theta_i = 0 \textrm{ if
}|\htheta_i| \le \thresh\}$.  This algorithm is:
\begin{align}\label{eq:refopt}
\reftheta = \argmin_{\theta \in \Theta_\Clip} \ \hat{\Loss}(\theta) + \lambda \OneNorm{\theta}
\end{align}

\begin{theorem}\label{thm:refit}
(Sparsity) Suppose that \ref{ass:RE} holds and the regularization parameter $\lambda$ satisfies both
\begin{equation} \label{eq:Blambda1} 
\InfNorm{\Expct[\staty] - \widehat \Expct[\staty]}
\ \leq \
\frac{\lambda}{2} ~ ~ \textrm{ and } ~ \lambda \leq \min\{\frac{1}{270  \salpha^2 \OneNorm{\stheta}}, \frac{\eigmin^2}{340 \eigmax \salpha \sqrt{\SupSize}} \}
\end{equation}
where $\salpha$ is the analytic standardized moment or cumulant of $\stheta$ for the subspace $\mathcal{V}$ defined in \eqref{eq:resdir}.
 If $\htheta$ is the solution of \eqref{eq:optimization} with this $\lambda$ and $\reftheta$ is the solution of \eqref{eq:refopt} with threshold $\thresh=\frac{18 \lambda}{\eigmin^2}$ and this $\lambda$,  then:
\begin{enumerate}
\item  $\reftheta$ has support on at most $2\SupSize$ coordinates.
\item The Fisher risk is bounded as follows:
\begin{align*}
\frac{1}{4} \FishNorm{\htheta - \stheta}^2  \leq \Loss(\htheta) -
\Loss(\stheta) \le \left( 12 \frac{\eigmax}{\eigmin}\right)^2 \frac{9\ \SupSize \lambda^2}{\eigmin^2}  
%\FishNorm{\reftheta - \stheta}^{2} \le  \left( \frac{\eigmax}{\eigmin}\right)^2 \frac{8464\ \SupSize \lambda^2}{\eigmin^2}  
\end{align*}
%\item We have the following $L_1$ risk bound
%\begin{align*}
%\OneNorm{\reftheta - \stheta} \le \left( \frac{\eigmax}{\eigmin}\right) \frac{127  \SupSize \lambda}{\eigmin^2}  
%\end{align*}
\end{enumerate}
\end{theorem}

Using Proposition \ref{prop:subgauss}, we have following corollary.
\begin{corollary}
Suppose the Assumption \ref{ass:RE} and the sub-Gaussian condition in Proposition \ref{prop:subgauss} hold. Then for any $\delta > 0$, as long as 
$
n \ge  K \salpha^2 \sigma^2 \log\left(\frac{p}{\delta}\right) \max\left\{\frac{\SupSize \eigmax^2}{\eigmin^4}, \salpha^2  \OneNorm{\stheta}^2  \right\} 
$ (where $K$ is a universal constant), 
setting $\lambda = 2 \sqrt{\tfrac{\sigma^2 \log\left(\frac{p}{\delta}\right)}{n}}$ and threshold $\thresh = 36 \sqrt{\tfrac{\sigma^2 \log\left(\frac{p}{\delta}\right)}{n \eigmin^2}}$, we have that with probability at least $1-\delta$,
\begin{align*}
 \FishNorm{\reftheta - \stheta}^2 \le  \left(12 \frac{\eigmax}{\eigmin}\right)^2 \left(\frac{36}{\eigmin^2}\right) \frac{\SupSize \sigma^2 \log\left(\frac{p}{\delta}\right)}{n} %\ \ \ \textrm{ and }\ \ \  \OneNorm{\reftheta - \stheta}  \leq 
%\frac{12 \SupSize \sigma \eigmax }{\eigmin^3}  \sqrt{\frac{\log(\frac{p}{\delta})}{n}}
\end{align*}
\end{corollary}

\iffalse
It is worth mentioning that $\salpha$ which is either of analytic standardized moment or analytic standardized cumulant for the subspace $\mathcal{V}$ defined in \eqref{eq:resdir} can be significantly better than the analytic moment or cumulant of set $\{\theta - \stheta : \theta \in \Theta\}$ (see appendix).
\fi

%% file: appendix.tex
\section{Appendix}

\subsection{Proofs for Section~\ref{section:strong}}

\begin{proof} (of Lemma~\ref{lemma:exist})
The proof shows that the central moment generating function of $z=\inner{v,t}$, namely $\Expct[\exp(s(\inner{v,t}-\Expct[\inner{v,t}]))]$, is analytic at $\theta$. First, notice that
\begin{eqnarray*}
\Expct[\exp(s(\inner{v,t}-\Expct[\inner{v,t}]))] &=& \exp(-s\Expct[\inner{v,t}])\int_{t}h_t \exp(s\inner{v,t})\exp\{\inner{\theta,t}-\log Z(\theta)\} dt\\
&=& \exp(-s\Expct[\inner{v,t}])\frac{\int_t h_t \exp\{\inner{\theta+sv,t}\}dt}{\int_t h_t \exp\{\inner{\theta,t}\}dt}\\
&=& \exp(-s\Expct[\inner{v,t}])\frac{Z(\theta+sv)}{Z(\theta)}.
\end{eqnarray*} 
It is known that for exponential families, $Z(\theta)$ (namely, the partition function) is analytic in the interior of $\Theta$ (see \cite{Brown86}). Since $\exp(-s\Expct[\inner{v,t}])$ is also analytic (as a function of $s$), we have by the chain of equalities above that the central moment generating function is also analytic (as a function of $s$) for any $\theta$ at the interior of $\Theta$. This property implies that the derivatives of the central moment generating function at $s=0$ (namely, the moments $\moment_{k,\dist}(z)$) cannot grow too fast with $k$. In particular, by proposition 2.2.10 in \cite{Krantz02}, it holds for all $k$ that the $k$-th derivative (which is equal to $\moment_{k,\dist}(z)$) is at most  $k!B^k$ for some constant $B$. As a result, $|\moment_{k,\dist}(z)/\moment_{2,\dist}(z)^{k/2}|$ is at most $\frac{1}{2}k!\alpha^{k-2}$ for a suitable constant $\alpha$. Thus, $t$ has finite analytic standardized moment with respect to all directions.

As to the assertion about $t$ having finite analytic standardized cumulant, notice that our argument above also implies that the (raw) moment generating function, $\Expct[\exp(s\inner{v,t})]$, is analytic. Therefore, $\log(\Expct[\exp(s\inner{v,t})])$, which is the cumulant generating function, is also analytic (since the logarithm is an analytic function). An analysis completely identical to the above leads to the desired conclusion about the cumulants of $t$.
\end{proof}

From here on, we slightly abuse notation and let $\moment_{k}(\Delta)$
be the $k$-th central moment of the univariate random variable
$\inner{\Delta,t}$ distributed under $\stheta$.

\begin{proof} (of Lemma~\ref{lemma:taylor})
First, note that since $\stheta$ is optimal, we have
$\Expct_{\staty\sim D}[\staty]=\Expct_{\staty\sim
  P(\cdot|\stheta)}[\staty]$. Hence, 
\begin{align*}
\Loss(\stheta+s\Delta)-\Loss(\stheta) 
&= -s\inner{\Delta, \Expct_{\staty\sim
  P(\cdot|\stheta)}[\staty]} +
\log\frac{Z(\stheta+s\Delta)}{Z(\stheta)}\\
&= -s\moment_1(\Delta) +
\log \frac{Z(\stheta+s\Delta)}{Z(\stheta)}\\
&= \log\frac{e^{-s\moment_1(\Delta)} Z(\stheta+s\Delta)}{Z(\stheta)}
\end{align*}
In the proof of Lemma \ref{lemma:exist} it was shown that $e^{-s\moment_1(\Delta)}
\frac{Z(\stheta+s\Delta)}{Z(\stheta)}$ is the central moment generating
function, that it is analytic, and that the expression above is analytic as well. Their Taylor expansions complete the proof.
\end{proof}

The following upper and lower bounds are useful in that they guarantee
the sum converges for the choice of $s$ specified. 

\begin{lemma} \label{lemma:upper_lower_s} Let $\alpha$ and $\theta$ be
  defined as in Theorem~\ref{thm:strong}. Let $\Delta=\theta-\stheta$ and set
  $s=\min\{\frac{1}{4\smoment \sqrt{\moment_{2}(\Delta)}},1\}$.  If
is $\alpha$ is an analytic moment, then
\begin{equation*}
\frac{1}{3} \frac{\moment_{2}(\Delta)}
{\max\{16\smoment^2\moment_{2}(\Delta),1\} }
\ \leq \
\sum_{k=2}^\infty \frac{\moment_{k}(\Delta) s^k}{k!} 
\ \leq \
\frac{2}{3} \frac{\moment_{2}(\Delta)}
{\max\{16\smoment^2\moment_{2}(\Delta),1\} }
\end{equation*}
If is $\alpha$ is an analytic cumulant, then
\begin{equation*}
\frac{1}{3} \frac{\cumulant_{2}(\Delta)}
{\max\{16\smoment^2\cumulant_{2}(\Delta),1\} }
\ \leq \
\sum_{k=2}^\infty \frac{\cumulant_{k}(\Delta) s^k}{k!} 
\ \leq \
\frac{2}{3} \frac{\cumulant_{2}(\Delta)}
{\max\{16\smoment^2\cumulant_{2}(\Delta),1\} }
\end{equation*}
\end{lemma}

\begin{proof}
We only prove the analytic moment case (the proof for the cumulant case is identical). First let us show that:
\[
\frac{s^2 \moment_2(\Delta) }{2}
\left(1 - \sum_{k=1}^\infty (s \smoment \sqrt{\moment_2(\Delta)})^k \right) 
\ \leq \
\sum_{k=2}^\infty \frac{\moment_{k}(\Delta) s^k}{k!} 
\ \leq \
\frac{s^2 \moment_2(\Delta) }{2}
\left(1 + \sum_{k=1}^\infty (s \smoment \sqrt{\moment_2(\Delta)})^k \right)
\]
We can bound the following sum from $k=3$ onwards as: 
\begin{align*}
\left|\sum_{k=3}^\infty \frac{1}{k!} \moment_{k}(\Delta) s^k \right| 
 \leq \frac{1}{2} \sum_{k=3}^\infty \smoment^{k-2} \moment_{2}(\Delta)^\frac{k}{2} s^k  
 = \frac{s^2\moment_{2}(\Delta)}{2} \sum_{k=1}^\infty (s \smoment \sqrt{\moment_{2}(\Delta)})^k 
\end{align*}
which proves the claim.

For our choice of $s$,
\begin{align*}
\sum_{k=1}^\infty (s \smoment \sqrt{\moment_{2}(\Delta)})^k 
= \sum_{k=1}^\infty \left(\min\left\{\frac{1}{4},\smoment
  \sqrt{\moment_{2}(\Delta)}\right\}\right)^k
\leq \sum_{k=1}^\infty \left(\frac{1}{4}\right)^k
= \frac{1}{3}
\end{align*}
Hence, we have:
\begin{align*}
\sum_{k=2}^\infty \frac{\moment_{k}(\Delta) s^k}{k!} 
& \geq \frac{s^2 \moment_{2}(\Delta) }{2}
\left(1 - \sum_{k=1}^\infty (s \smoment \sqrt{\moment_{2}(\Delta)})^k
\right)\\
& \geq \frac{s^2 \moment_{2}(\Delta) }{3}\\
& = \frac{1}{3} \frac{\moment_{2}(\Delta)}
{\max\{16\smoment^2\moment_{2}(\Delta),1\} }\\
\end{align*}
Analogously, the upper bound can be proved. 
\end{proof}

The following core lemma leads to the proof of Theorem~\ref{thm:strong}.

\begin{lemma} \label{lemma:upper_lower} Let $\alpha$ and $\theta$ be
  defined as in Theorem~\ref{thm:strong}. We
  have that:
\begin{equation}\label{eq:lower_tmp}
\frac{1}{4} \ \frac{\FishNorm{\theta-\stheta}^2}
{\max\{16\smoment^2\FishNorm{\theta-\stheta}^2,1\} }
\ \leq \
\Loss(\theta)-\Loss(\stheta) 
%\ \leq \
%\frac{2}{3} \ \frac{\Risk(\theta)}
%{\max\{16\smoment^2\Risk(\theta),1\} }
\end{equation}
Furthermore, if $\FishNorm{\theta-\stheta} \leq \frac{1}{16\alpha^2}$,
\[
\frac{1}{4} \FishNorm{\theta-\stheta}^2 
\ \leq \
\Loss(\theta)-\Loss(\stheta) 
\ \leq \
\frac{2}{3} \FishNorm{\theta-\stheta}^2
\]
\end{lemma}

\begin{proof}
As $s$ is clearly in $[0,1]$ and by
  convexity, we have:
\begin{align*}
\Loss(\theta)-\Loss(\stheta) 
&= \Loss(\stheta+\Delta)-\Loss(\stheta) \\
& \geq \Loss(\stheta+s\Delta)-\Loss(\stheta) \\
\end{align*} 
For the cumulant case, we have that this is lower bounded by 
$\frac{\moment_{2}(\Delta)}
{3 \max\{16\smoment^2\moment_{2}(\Delta),1\} }$ using
Lemma~\ref{lemma:upper_lower_s} and Lemma~\ref{lemma:taylor}, which
proves \eqref{eq:lower_tmp}. Now consider the analytic moment case. By, Lemma~\ref{lemma:taylor},
we have
\begin{align*}
\Loss(\theta)-\Loss(\stheta) 
 \geq \log(1+\frac{\moment_{2}(\Delta)}
{3 \max\{16\smoment^2\moment_{2}(\Delta),1\} })
\end{align*}
Now by Jensen's inequality,
we know that the fourth standardized moment (the kurtosis) is greater
than one, so $\alpha^2 \geq \frac{1}{12}$ (since  $\frac{4!}{2} \alpha^2\geq 1$). This implies
that:
\[
\frac{\moment_{2}(\Delta)}
{3 \max\{16\smoment^2\moment_{2}(\Delta),1\} } \leq
\frac{1}{48\alpha^2} \leq 1/4
\]
since the sum is only larger if we choose any argument in the $\max$. Now
for $0\leq x\leq 1/4$, we have that $\log(1+x)\geq 1+x-x^2\geq
1+\frac{3}{4}x$. Proceeding, 
\begin{equation*}
\log(1+\frac{\moment_{2}(\Delta)}
{3 \max\{16\smoment^2\moment_{2}(\Delta),1\} }) \geq
\frac{\moment_{2}(\Delta)}
{4 \max\{16\smoment^2\moment_{2}(\Delta),1\} }
\end{equation*}
which proves \eqref{eq:lower_tmp} (for the analytic moment case). 

For the second claim, the precondition implies that the max, in
\eqref{eq:lower_tmp}, will be 
achieved with the argument of $1$, which directly implies the lower
bound. For the upper bound, we can apply
Lemma~\ref{lemma:upper_lower_s} with $s=1$ ($s=1$ under our
precondition), which
 implies that
$\sum_{k=2}^\infty \frac{\moment_{k}(\Delta)}{k!} $ is less than 
$\frac{2}{3} \moment_{2}(\Delta)$. The claim follows directly for the
cumulant case using Lemma~\ref{lemma:taylor}, with $s=1$. For the moment case, we
use that $\log(1+x)\leq x$. 
\end{proof}

We are now ready to prove Theorem~\ref{thm:strong}.

\begin{proof} (of Theorem~\ref{thm:strong})
If $\FishNorm{\theta-\stheta}^2 \leq
\frac{1}{16\smoment^2}$, then the previous Lemma implies the claim.
Let us assume the condition on the loss,
i.e. $\Loss(\theta)-\Loss(\stheta) \leq \frac{1}{65 \smoment^2}$. If
$\FishNorm{\theta-\stheta}^2 \leq \frac{1}{16\smoment^2}$, then we are done by the previous
argument. So let us assume that $\FishNorm{\theta-\stheta}^2 > \frac{1}{16\smoment^2}$.
Hence, $\max\{16\smoment^2\moment_{2}(\Delta),1\} =
16\smoment^2\moment_{2}(\Delta)$. Using \eqref{eq:lower_tmp}, we have that
$\frac{1}{64 \smoment^2} \leq \Loss(\theta)-\Loss(\stheta)$, which is
a contradiction.
\end{proof}

\subsection{Proofs for Section~\ref{section:sparse}}

\subsubsection{Proof of Theorem~\ref{thm:risk}}

Throughout, let $\ELoss(\theta)=\widehat\Expct[-\log
P(y|\theta)]$. Also, let $T= \Expct[\staty]$ and $\hat T = \widehat
\Expct[\staty]$. 

\begin{lemma} \label{lemma:reg_properties1} Suppose that
\eqref{eq:lambda} holds (i.e. that $\InfNorm{T-\hat T} \leq
\lambda /2$). Let $\htheta$ be a solution the optimization problem in
\eqref{eq:optimization}. For all $\theta \in \Theta$, we have: 
\begin{align} \label{eq:regret_bound1}
\Loss(\htheta)-\Loss(\theta) 
& \leq \frac{\lambda}{2}\OneNorm{\htheta - \theta}  
+ \lambda \OneNorm{\theta} - \lambda \OneNorm{\htheta}  \\
&\leq \frac{3\lambda}{2} \OneNorm{\theta } \notag
\end{align}
Furthermore, suppose that $\theta$ only has support on $\Sup$,  then:
\begin{equation} \label{eq:regret_bound2}
\Loss(\htheta)-\Loss(\theta)  \leq \frac{3\lambda}{2}
\OneNorm{\htheta_\Sup-\theta }
\end{equation}
\end{lemma}
\begin{proof}
Since $\htheta$ solves \eqref{eq:optimization}, we have:
\[
-\inner{\htheta,\hat T}+\log Z(\htheta) + \lambda
\OneNorm{\htheta} \leq
-\inner{\theta,\hat T}+\log Z(\theta) + \lambda \OneNorm{\theta}
\]
Hence,
\[
-\inner{\htheta,T}+\log Z(\htheta) + \lambda
\OneNorm{\htheta} 
\leq 
\inner{ \htheta - \theta,\hat T - T}
-\inner{\theta, T}+\log Z(\theta) + \lambda\OneNorm{\theta}
\]
Using this and the condition on $\lambda$, we have
\begin{align*}
\Loss(\htheta)-\Loss(\theta)  
& \leq \inner{ \htheta - \theta,\hat T - T} + \lambda
\OneNorm{\theta} - \lambda \OneNorm{\htheta}  \\ 
& \leq \OneNorm{\htheta - \theta} \InfNorm{\hat T - T} + \lambda
\OneNorm{\theta} - \lambda \OneNorm{\htheta}  \\
& \leq \frac{\lambda}{2}\OneNorm{\htheta - \theta}  
+ \lambda \OneNorm{\theta} - \lambda \OneNorm{\htheta} 
\end{align*}
which proves the first inequality. Continuing,
\begin{align*}
& \frac{\lambda}{2}\OneNorm{\htheta - \theta}  
+ \lambda \OneNorm{\theta} - \lambda \OneNorm{\htheta} \\
\leq &
\frac{\lambda}{2}(\OneNorm{\htheta} +\OneNorm{\theta})
+ \lambda \OneNorm{\theta} - \lambda \OneNorm{\htheta} \\
\leq & \frac{3\lambda}{2} \OneNorm{\theta } 
\end{align*}
which proves the next inequality.

For the final claim, using the sparsity assumption on $\theta$, we have:
\begin{align*}
\Loss(\htheta)-\Loss(\theta) 
& \leq \frac{\lambda}{2}\OneNorm{\htheta - \theta} + \lambda
\OneNorm{\theta} - \lambda \OneNorm{\htheta} \\
& = \frac{\lambda}{2}\OneNorm{\htheta_\Sup-\theta} +
\frac{\lambda}{2}\OneNorm{\htheta_\SupC} +
\lambda \left(\OneNorm{\theta}-\OneNorm{\htheta_\Sup}\right) 
- \lambda\OneNorm{\htheta_\SupC} \\
& \leq \frac{\lambda}{2}\OneNorm{\htheta_\Sup-\theta} +
\lambda \OneNorm{\htheta_\SupC} +
\lambda \OneNorm{\htheta_\Sup-\theta}
- \lambda\OneNorm{\htheta_\SupC} \\
& = \frac{3\lambda}{2}\OneNorm{\htheta_\Sup-\theta} 
\end{align*}
where the second to last step uses the triangle inequality. 
This completes the proof.
\end{proof}

\begin{lemma} \label{lemma:reg_properties2}
Suppose that
\eqref{eq:lambda} holds. Let $\htheta$ be a solution the optimization problem in
\eqref{eq:optimization}. 
For any $\theta \in \Theta$, which only has support on $\Sup$ and such that $\Loss(\htheta)\geq \Loss(\theta)$,  then:
\begin{align}
\OneNorm{\htheta_\SupC} & \leq  3 \OneNorm{\htheta_\Sup-\theta }
\label{eq:mass_condition}\\
\OneNorm{\htheta-\theta } & \leq 4 \OneNorm{\htheta_\Sup-\theta } 
\label{eq:mass_condition2}
\end{align}
\end{lemma}

\begin{proof}
By assumption on $\theta$ and \eqref{eq:regret_bound1},
\[
0 \leq  \Loss(\htheta)-\Loss(\theta)  
\leq \frac{\lambda}{2}\OneNorm{\htheta - \theta}  
+ \lambda \OneNorm{\theta} - \lambda \OneNorm{\htheta} 
\]
Dividing by $\lambda$ and adding $\frac{1}{2}\OneNorm{\htheta -
  \theta}$ to both the left and right sides, 
\[
\frac{1}{2}\OneNorm{\htheta - \theta} \leq \OneNorm{\htheta - \theta}  +  \OneNorm{\theta} -  \OneNorm{\htheta}
\]
For any component $i\notin S$, we have that  $|\htheta_i -
\theta_i|+|\theta_i| -  |\htheta_i|=0$. Hence,
\[
\frac{1}{2}\OneNorm{\htheta - \theta} \leq \OneNorm{\htheta_\Sup
  - \theta}  +  \OneNorm{\theta} -  \OneNorm{\htheta_\Sup}
\leq 2 \OneNorm{\htheta_\Sup - \theta}
\]
 where the last step uses the triangle inequality ($\OneNorm{\theta} -  \OneNorm{\htheta_\Sup}
\leq \OneNorm{\htheta_\Sup - \theta}$). This proves
\eqref{eq:mass_condition2}. From this, 
\[
\frac{1}{2}\OneNorm{\htheta_\Sup -
  \theta}+\frac{1}{2}\OneNorm{\htheta_\SupC}
=\frac{1}{2}\OneNorm{\htheta - \theta} 
\leq 2 \OneNorm{\htheta_\Sup - \theta}
\]
which proves \eqref{eq:mass_condition}, after rearranging.
\end{proof}

Now we are ready to prove Theorem~\ref{thm:risk}.

\begin{proof} (of Theorem~\ref{thm:risk}).
First, by \eqref{eq:lambda} and \eqref{eq:regret_bound1} we see that
\begin{align*}
\Loss(\htheta) - \Loss(\stheta) \le \frac{1}{65 \salpha^2}
\end{align*}
(note that $\htheta$ satisfies the RE precondition, so $\htheta-\stheta\in\mathcal{V}$).
Hence using Theorem~\ref{thm:strong} we see that
\begin{align*}
\frac{1}{4} \FishNorm{\htheta - \stheta}^2 \le \Loss(\htheta) - \Loss(\stheta)
\end{align*}
On the other hand observe that:
\begin{equation} \label{eq:OneBound_S}
\OneNorm{\htheta_\Sup-\stheta }
 \leq \sqrt{s} \TwoNorm{\htheta_\Sup-\stheta }
 \leq \frac{\sqrt{s}}{\eigmin}
\FishNorm{\htheta-\stheta}
\end{equation}
where the last step uses the Restricted Eigenvalue Condition,
Assumption~\ref{ass:RE}. Now using the above with \eqref{eq:regret_bound2} we have that
\begin{align*}
\frac{1}{4} \FishNorm{\htheta - \stheta}^2 \le \Loss(\htheta) - \Loss(\stheta) \le \frac{3 \lambda \sqrt{s}}{2 \eigmin}
\FishNorm{\htheta-\stheta}
\end{align*}
Hence,
\begin{equation}~\label{eq:fisher_tmp}
\FishNorm{\htheta - \stheta} \le \frac{6 \lambda \sqrt{s}}{ \eigmin}
\end{equation}
and so
\begin{align*}
\frac{1}{4} \FishNorm{\htheta - \stheta}^2 \le \Loss(\htheta) -
\Loss(\stheta) \le \frac{9 \lambda^2 s}{\eigmin^2} \ 
\end{align*}
which proves the first claim. 

Now to conclude the proof note that by Assumption~\ref{ass:RE}
$$
\eigmin \TwoNorm{\htheta_\Sup - \stheta} \le \FishNorm{\htheta - \stheta} \le \frac{6 \lambda \sqrt{s}}{ \eigmin}
$$
Hence by \eqref{eq:mass_condition2} we see that
$$
\OneNorm{\htheta - \stheta} \le 4 \OneNorm{\htheta_\Sup - \stheta}\le 4 \sqrt{s} \TwoNorm{\htheta_\Sup - \stheta} \le  \frac{24 \lambda s}{ \eigmin^2}
$$
This concludes the proof.
\end{proof}

\subsubsection{Analytic Standardized Moment for GLM and Sparsity}\label{sec:glmsparse}
In the generalized linear model example in Section \ref{sec:glm}, we
showed that if the sufficient statistics are bounded by
$B$ and if $\sFisher$ has minimum
eigenvalue $\lambda_\mathrm{min}$, then we can choose $\smoment =
B/\lambda_\mathrm{\min}$. However, when $\stheta$ is sparse we see that
in both Theorems \ref{thm:risk} and \ref{thm:refit}, we only care
about $\salpha$ the analytic standardized moment/cumulant of the set
$\mathcal{V}$, specified in \eqref{eq:resdir}. Given this, it
is clear from the exposition in
the generalized linear model example in Section \ref{sec:glm} that
$\salpha$ can be bounded by $B/\eigmin$, since all elements of the set
$\mathcal{V}$ satisfy Assumption \ref{ass:RE}.

\subsubsection{Proof of Theorem \ref{thm:refit}}
\begin{lemma}(Sparsity or Restricted Set) \label{lem:sparse}
If the threshold $\thresh = \frac{18 \lambda}{\eigmin^2}$, then the size of the support of any $\theta \in \Theta_\Clip$ is at most $2\SupSize$
\end{lemma}
\begin{proof}
  First notice that on the set $\Sup$ thresholding could potentially
  leave all the $s$ coordinates. On the other hand notice
  that if we  threshold using $\thresh$, then the number of coordinates
  that remain unclipped in the set $\SupC$ is bounded by
  $\OneNorm{\htheta_{\SupC}}/\thresh$. Hence
$$
\left| i : |\htheta_i| > \thresh\right| \le  \SupSize + \frac{\|\htheta_{\SupC}\|_1}{\thresh}
$$
By \eqref{eq:mass_condition}, \eqref{eq:fisher_tmp} and the RE assumption, we have
$$
\OneNorm{\htheta_{\SupC}} \le 3 \OneNorm{\htheta_\Sup - \stheta}\le 3 \sqrt{\SupSize} \TwoNorm{\htheta_\Sup - \stheta} \le  \frac{18 \lambda \SupSize}{\eigmin^2}
$$
Using this we see that 
$$\left| i : |\htheta_i| > \thresh\right| \le  \SupSize + \frac{18 \lambda \SupSize}{ \eigmin^2 \thresh}$$
Plugging in the value of $\thresh$ we get the statement of the lemma since support size of $\cliptheta$ upper bounds the support size of any $\theta \in \Theta_\Clip$.
\end{proof}

\begin{lemma}(Bias)  \label{lem:clip}
Choose $\thresh = \frac{18 \lambda}{\eigmin^2}$. Then,
$$
\Loss(\cliptheta_ \Sup) - \Loss(\stheta) \le \frac{540 \eigmax^2 \SupSize\lambda^2}{ \eigmin^4} 
$$
where $\cliptheta$ is defined as $\cliptheta_i = \htheta_i \mathbf{1}_{(\htheta_i > \thresh)}$.
\end{lemma}
\begin{proof}
Note that 
\begin{align*}
\FishNorm{\cliptheta_{ \Sup} - \stheta}^2 & \le \eigmax^2 \|\cliptheta_{\Sup} - \stheta\|^2_2\\
& \le 2 \eigmax^2  \left(\|\cliptheta_{\Sup} - \htheta_S\|^2_2  + \|\htheta_\Sup - \stheta\|^2_2\right)\\
& \le 2 \eigmax^2  \left( \SupSize \thresh^2 + \|\htheta_\Sup - \stheta\|^2_2\right)\\
& \le 2 \eigmax^2  \left( \SupSize \thresh^2 + \frac{36 \SupSize\lambda^2}{\eigmin^4} \right) 
\end{align*}
Where the last step is obtained by applying Theorem
\ref{thm:risk}. Substituting for $\thresh$,
\begin{align}\label{eq:clipt}
\FishNorm{\cliptheta_{S} - \stheta}^2 & \le    \frac{720 \eigmax^2 \SupSize\lambda^2}{ \eigmin^4}
\end{align}
Now the condition on $\lambda$ in \eqref{eq:Blambda1} implies
that Theorem \ref{thm:strong} is applicable, which completes the
proof.
\end{proof}

\begin{proof}[Proof of Theorem \ref{thm:refit}]
The first claim of the theorem follows from Lemma \ref{lem:sparse}. We prove the second claim of the theorem  by considering two cases. First, when $\Loss(\reftheta) \le \Loss(\cliptheta_\Sup)$. In this case by Lemma \ref{lem:clip} we have
$$
\Loss(\reftheta) - \Loss(\stheta)  \le \frac{540 \eigmax^2 \SupSize\lambda^2}{ \eigmin^4} 
$$
Also by \eqref{eq:Blambda1}, applying Theorem \ref{thm:strong}, we see that 
$$
\frac{1}{4}\FishNorm{\reftheta - \stheta}^2 \le \Loss(\reftheta) - \Loss(\stheta)  \le \frac{540 \eigmax^2 \SupSize\lambda^2}{ \eigmin^4} 
$$
which gives us the second claim of the theorem.The next case is when  $\Loss(\reftheta) > \Loss(\cliptheta_\Sup)$. In this case, by applying Lemma \ref{lemma:reg_properties1} with $\theta = \cliptheta_\Sup$, we see that 
\begin{align*}
\Loss(\reftheta) - \Loss(\cliptheta_\Sup)  \le \frac{3 \lambda}{2} \OneNorm{\cliptheta_\Sup} & \le \frac{3 \lambda}{2} \OneNorm{\stheta - \cliptheta_\Sup} + \frac{3 \lambda}{2} \OneNorm{\stheta} \\
& \le  \frac{3 \lambda \sqrt{\SupSize}}{2} \TwoNorm{\stheta - \cliptheta_\Sup} + \frac{3 \lambda}{2} \OneNorm{\stheta} \\
& \le  \frac{3 \lambda \sqrt{\SupSize}}{2 \eigmin} \FishNorm{\stheta - \cliptheta_\Sup} + \frac{3 \lambda}{2} \OneNorm{\stheta}\\
& \le \frac{18 \sqrt{5} \lambda^2 \SupSize \eigmax}{\eigmin^3}  + \frac{3 \lambda}{2} \OneNorm{\stheta}
\end{align*}
where the last step is using \eqref{eq:clipt}. Hence we see that
\begin{align*}
\Loss(\reftheta) - \Loss(\stheta) & \le \Loss(\reftheta) - \Loss(\cliptheta_\Sup) + \Loss(\cliptheta_\Sup)  - \Loss(\stheta) \le \frac{581 \eigmax^2 \SupSize\lambda^2}{ \eigmin^4} + \frac{3 \lambda}{2} \OneNorm{\stheta} 
%\Loss(\reftheta) - \Loss(\stheta) & \le \Loss(\reftheta) - \Loss(\cliptheta_\Sup)  + \Loss(\reftheta) - \Loss(\cliptheta_\Sup) \le \frac{851 \eigmax^2 \SupSize\lambda^2}{ \eigmin^4} + \frac{3 \lambda}{2} \OneNorm{\stheta} 
\end{align*}
Hence by condition (\ref{eq:Blambda1}) on $\lambda$ we see that the pre-condition of the Theorem \ref{thm:strong} is satisfied and hence we see that
\begin{align}
\frac{1}{4} \FishNorm{\reftheta - \stheta }^2 \le \Loss(\reftheta) - \Loss(\stheta)  & \le \Loss(\reftheta) - \Loss(\cliptheta_\Sup) + \Loss(\cliptheta_\Sup)  - \Loss(\stheta)  \notag \\
& \le \Loss(\reftheta) - \Loss(\cliptheta_\Sup) + \frac{540\eigmax^2 \SupSize\lambda^2}{ \eigmin^4}\notag  \\
& \le \frac{3 \lambda}{2} \OneNorm{\reftheta - \cliptheta_\Sup } + \frac{540\eigmax^2 \SupSize \lambda^2}{ \eigmin^4} \label{eq:long1}\\
 & \le  6 \lambda \OneNorm{\reftheta_\Sup - \cliptheta_\Sup } + \frac{540\eigmax^2 \SupSize\lambda^2}{ \eigmin^4} \label{eq:long2}\\
& \le  6 \lambda \sqrt{\SupSize} \TwoNorm{\reftheta_\Sup - \cliptheta_\Sup } + \frac{540\eigmax^2 \SupSize\lambda^2}{ \eigmin^4} \notag \\
 & \le  \frac{6 \lambda \sqrt{\SupSize}}{\eigmin} \FishNorm{\reftheta - \cliptheta_\Sup } + \frac{540\eigmax^2 \SupSize\lambda^2}{ \eigmin^4} \label{eq:long3} \\
& \le  \frac{6 \lambda \sqrt{\SupSize}}{\eigmin} \FishNorm{\reftheta - \stheta} + \frac{6 \lambda \sqrt{\SupSize}}{\eigmin} \FishNorm{\stheta - \cliptheta_\Sup } + \frac{540\eigmax^2 \SupSize\lambda^2}{ \eigmin^4} \notag  \\
&  \le  \frac{6 \lambda \sqrt{\SupSize}}{\eigmin} \FishNorm{\reftheta - \stheta} + \frac{161 \eigmax  \SupSize \lambda^2 }{\eigmin^2}  + \frac{540\eigmax^2 \SupSize\lambda^2}{ \eigmin^4} \label{eq:long4} 
\end{align}
Where \eqref{eq:long1} is obtained by applying Lemma \ref{lemma:reg_properties1} on $\Theta = \Theta_\Clip$ and \eqref{eq:long2} is by Lemma \ref{lemma:reg_properties2} with $\Theta = \Theta_\Clip$. \eqref{eq:long3} is by Assumption \ref{ass:RE} and  \eqref{eq:long4} is due to \eqref{eq:clipt}. Simplifying we conclude that
\begin{align}\label{eq:normineq}
\frac{1}{4} \FishNorm{\reftheta - \stheta }^2 \le \Loss(\reftheta) - \Loss(\stheta)   \le  \frac{6 \lambda \sqrt{\SupSize}}{\eigmin} \FishNorm{\reftheta - \stheta} +  \frac{701 \eigmax^2 \SupSize\lambda^2}{ \eigmin^4} 
\end{align}
By the  inequality that for any $a,b \in \mathbb{B}$, $ab \le \frac{a^2}{2} + \frac{b^2}{2}$ we have
$$
\frac{1}{2}\FishNorm{\reftheta - \stheta }^2  \le  \frac{288 \lambda^2 \SupSize}{\eigmin^2} + \frac{2804 \eigmax^2 \SupSize\lambda^2}{ \eigmin^4} 
$$
Thus  
$$
\FishNorm{\reftheta - \stheta } \le  \frac{24 \lambda \sqrt{\SupSize}}{\eigmin}  + \frac{75 \eigmax \lambda \sqrt{\SupSize}}{ \eigmin^2} 
$$
Using this in \eqref{eq:normineq} 
\begin{align*}
\Loss(\reftheta) - \Loss(\stheta) \le  \frac{144 \lambda^2 \SupSize}{\eigmin^2} + \frac{450 \eigmax \lambda^2 \SupSize}{ \eigmin^3}  +  \frac{701 \eigmax^2 \SupSize\lambda^2}{ \eigmin^4} 
\end{align*}
Simplifying we get the second claim of the theorem for the second case. 
\end{proof}

%% file: errata.tex
\section{Errata}

\subsection{Cumulants of the Bernoulli Distribution}

In Section~\ref{sec:onedim}, we claimed that the cumulants $\cumulant_k(\stheta)$ of the Bernoulli distribution satisfy $\cumulant_k(\stheta) \le \cumulant_2(\stheta) = \moment_2(\stheta)$ for $k \ge 3$.
This claim is incorrect. Thanks to Francis Bach for pointing this out to us (personal email communication, 2015). However, all we needed was the existence of what we call an \emph{analytic standardized cumulant}. The following lemma
suffices to prove that one exists for the Bernoulli distribution. The result below is very likely to be classical. In any case, it follows easily from classical results on cumulants. We provide a proof below for completeness.

\begin{lemma}
The cumulants of the Bernoulli distribution satisfy, for $k \ge 3$:
\[
|\cumulant_k(\stheta)| \le (k-1)! \cdot \cumulant_2(\stheta) .
\]
\end{lemma}
\begin{proof}
Let us work with the mean parameter $p = \moment_1(\stheta)$. It is well known\footnote{See, for example, Eq. (3.3.12) on p. 56 of the book \emph{Introduction to Statistical Inference} (Dover Publications, 1995) by E. S. Keeping.} that
\[
\cumulant_{k+1}(p) = p \cdot (1-p) \cdot \cumulant_k'(p)
\]
where $\cumulant_k'(p)$ is the derivative of $\cumulant_k(p)$ w.r.t. $p$.

We first prove, by induction on $k \ge 2$, that $\cumulant_{k}(p)$ is a polynomial of degree $k$ in $p$ with $k$ real roots in the interval $[0,1]$, two of which are $0$ and $1$.
Claim is true for $k=2$ since $\cumulant_2(p) = p(1-p)$. If $c_k(p)$ has $k$ reals roots in $[0,1]$ then $c_k'(p)$ has
$k-1$ reals roots in $[0,1]$. This is because, by the Gauss-Lucas Theorem, roots of the derivative of a polynomial are in the convex hull of the roots of the polynomial itself.
This immediately implies that $\cumulant_{k+1}(p) = p(1-p)\cumulant_k'(p)$ has $k+1$ real roots in $[0,1]$ given that the extra factor $p(1-p)$ has roots $0$ and $1$.

Given the claim above, we can express $\cumulant_k(p)$ as
\[
\cumulant_k(p) = a_k \cdot p \cdot (1-p) \cdot \prod_{i=1}^{k-2} (p - r_i)
\]
for some $a_k \in \R$ and $r_i \in [0,1]$. Note that $a_2 = 1$ and because $\cumulant_{k+1}(p) = p(1-p)\cumulant_k'(p)$, we also have $|a_{k+1}| = k |a_k|$.
Therefore, $|a_k| = (k-1)!$ for all $k \ge 2$. The lemma now follows because
\begin{align*}
| \cumulant_k(p) | &= | a_k | \cdot p \cdot (1-p) \cdot \left| \prod_{i=1}^{k-2} (p - r_i) \right| \\
&\le (k-1)! \cdot \cumulant_2(p) \cdot \prod_{i=1}^{k-2} | p-r_i | \\
&\le (k-1)! \cdot \cumulant_2(p) .
\end{align*}
Note that since $p,r_i \in [0,1]$, we have $|p-r_i| \le 1$. 
\end{proof}